\documentclass[letterpaper, 10 pt, conference]{ieeeconf}  

\IEEEoverridecommandlockouts                              

\usepackage{graphics} 
\usepackage{epsfig} 
\usepackage{times} 
\usepackage{amsmath} 
\usepackage{amssymb}  

\usepackage{amsthm}  

\usepackage{url}
\usepackage{hyperref}
\hypersetup{
    colorlinks=true, 
    urlcolor=cyan,
    }
\usepackage{booktabs}
\usepackage[ruled, vlined, linesnumbered]{algorithm2e}

\usepackage{color}
\usepackage{xcolor}

\newtheorem{theorem}{Theorem}

\newtheorem{lemma}{Lemma}

\newtheorem{observation}{Observation}

\title{\LARGE \bf
Autonomous Navigation in Ice-Covered Waters with Learned Predictions on Ship-Ice Interactions}

\author{Ninghan Zhong$^{1}$, Alessandro Potenza$^{2}$, Stephen L. Smith$^{1}$
\thanks{This work is supported by the National Research Council Canada (NRC).} 
\thanks{$^{1}$Department of Electrical and Computer Engineering, University of Waterloo, Waterloo, ON N2L 3G1, Canada (e-mail: \protect\url{{n5zhong, stephen.smith}@uwaterloo.ca})}
\thanks{$^{2}$Department of Electrical and Computer Engineering, University of Manitoba, Winnipeg, MB R3T 2N2, Canada (e-mail: \protect\url{potenzaa@myumanitoba.ca}).  This work was performed while Alessandro Potenza was visiting the University of Waterloo.}
}

\begin{document}

\maketitle
\thispagestyle{empty}
\pagestyle{empty}

\begin{abstract}
\label{sec:abstract}
Autonomous navigation in ice-covered waters poses significant challenges due to the frequent lack of viable collision-free trajectories. When complete obstacle avoidance is infeasible, it becomes imperative for the navigation strategy to minimize collisions. Additionally, the dynamic nature of ice, which moves in response to ship maneuvers, complicates the path planning process. To address these challenges, we propose a novel deep learning model to estimate the coarse dynamics of ice movements triggered by ship actions through occupancy estimation. To ensure real-time applicability, we propose a novel approach that caches intermediate prediction results and seamlessly integrates the predictive model into a graph search planner. We evaluate the proposed planner both in simulation and in a physical testbed against existing approaches and show that our planner significantly reduces collisions with ice when compared to the state-of-the-art. Codes and demos of this work are available at {\footnotesize \url{https://github.com/IvanIZ/predictive-asv-planner}}.
\end{abstract}
\section{Introduction}
\label{sec:introduction}
In recent years, the polar regions have been attracting international attention. The Arctic area offers shorter shipping routes and rich natural resources while tourism in Antarctica is on the rise~\cite{sea_ice_warning}. However, navigating in ice-covered waters such as the Arctic areas poses high risks due to significantly higher ice concentrations compared to typical maritime environments~\cite{analysis_ship_operation}. 
Recent advancement of autonomous surface vehicles (ASVs) holds promise for safer and more efficient navigation in these icy waters. Nevertheless, current approaches still face substantial challenges. The high concentration of ice often makes a collision-free path infeasible. Additionally, the dynamic and chaotic nature of ice movements in response to ship maneuvers adds another layer of complexity to the task~\cite{rod2023ICRA}.

In this paper, we address the problem of path planning for an autonomous surface vehicle in ice-covered waters. 
We assume the ASV is designed for standard maritime operations and has limited ship-ice collision protection~\cite{rod2023ICRA}. 
Our work aims to compute a reference path for the ASV such that both ship-ice collisions and traveling distance are minimized. 

Existing works on ASV path planning primarily focused on finding a collision-free path~\cite{optimization_based_planner, RRT_COLREGS, velocity_obstacle}, and are not easily generalizable to environments with high ice concentrations, where collision-free paths are typically non-existent, as illustrated in Fig.~\ref{fig:sim_fig_front}. While the works in~\cite{rod2023ICRA} and~\cite{skeleton_planner} address this challenge, they do not consider ice motion during planning.
As ice concentration increases, the motion of the ice becomes more complex due to interactions with the ship and other ice floes, resulting in a more dynamic environment. Consequently, planners that treat ice floes as static objects become less effective. For example, the performance improvement seen with the planner proposed in~\cite{rod2023ICRA} diminishes with 50\% ice concentration.
In contrast, our approach incorporates predicted ice motion into the path planning process, ensuring robustness in both high and low ice concentration environments.

\setlength{\belowcaptionskip}{-5pt}
\begin{figure}[t]
    \centering
    \includegraphics[width=\linewidth]{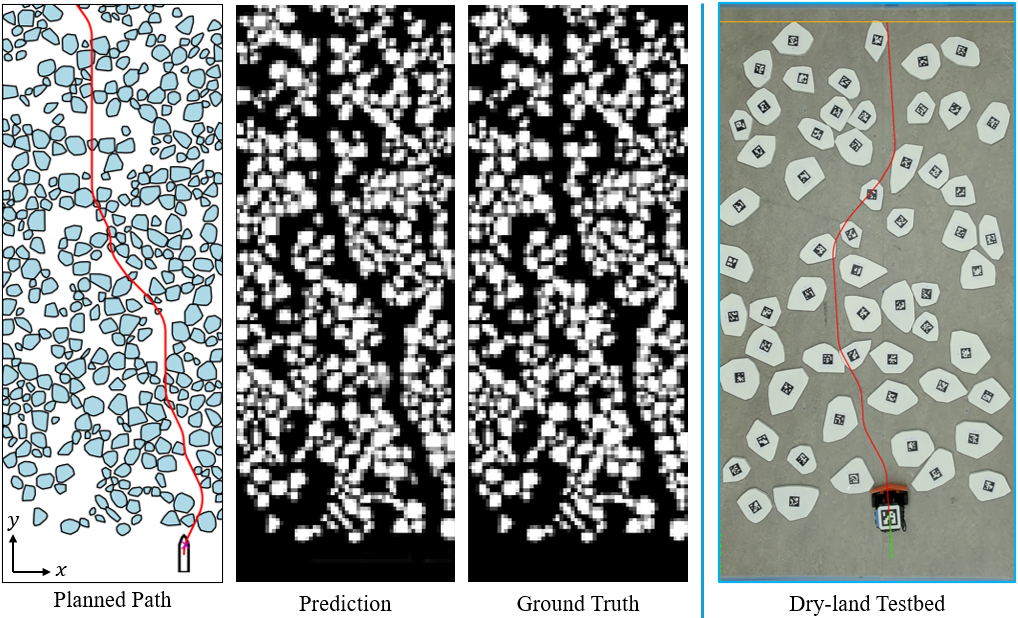}
    \caption{The left panel shows a planned path from the proposed planner in simulation. The middle panels compare the predicted occupancy after path traversal with the ground truth occupancy. The right panel presents an experimental trial from the ship-ice navigation testbed at the Autonomous Systems Laboratory, University of Waterloo.}
    \label{fig:sim_fig_front}
\end{figure}
\setlength{\belowcaptionskip}{1ex}

\paragraph{Contributions} In this paper, we propose a deep learning framework that predicts the coarse dynamics of the ice motion in response to the ship actions through occupancy estimation. The learning process is guided with a novel physics-derived loss function tailored to our occupancy formulation. 
To leverage the prediction results for planning, we present a simple yet empirically effective cost function based on occupancy maps to penalizes collisions. Further, we propose a graph search planner that seamlessly incorporates the learned model by caching the intermediate prediction results. Finally, the proposed planner is evaluated both in simulation and in a physical testbed, and superior performance is demonstrated. 

\paragraph{Related Work} 
%
While path planning for ASV ice navigation has been extensively studied, much of the existing literature focuses on computing a path on a global scope~\cite{arctic_planning_uncertain_model, finding_save_efficient_routs, auto_passage_planning_polar}. For instance, the planned paths in~\cite{arctic_planning_uncertain_model} and~\cite{finding_save_efficient_routs} span hundreds of kilometers while the paths considered in~\cite{auto_passage_planning_polar} take days to traverse. While these works demonstrated good performance, local planners are required to achieve full ship autonomy.  In ASV local planning, however, most work aims to compute a collision-free path~\cite{optimization_based_planner, RRT_COLREGS, velocity_obstacle, ice_nav_radar_image}, which does not generalize to environments where a collision-free path does not exist. 

Perhaps the most comparable works to this paper are~\cite{skeleton_planner} and~\cite{rod2023ICRA}. In~\cite{skeleton_planner}, the authors leverage morphological skeletons to represent the open-water areas given an overhead image of the ice field. An A* algorithm is then applied to find a path in the resulting graph. In~\cite{rod2023ICRA}, a lattice-based planner with a kinetic-energy based cost function is proposed to generate paths with minimal collisions. While \cite{skeleton_planner} and \cite{rod2023ICRA} apply to our settings, the authors treat obstacles as static during planning.

\emph{Ice Prediction Models:} Ice motion predictions, together with ice-ice and ship-ice interactions, have been primarily studied with numerical methods and empirical formulas~\cite{predict_CFD_DEM, predict_wave_induced, predict_circular_floe, predict_empirical_equation}. While these approaches demonstrate high accuracy, numerical simulations are computationally intensive, limiting their real-time applicability. On the other hand, empirical formulas are computationally efficient, but  often require environmental data that is not available online from ASV onboard sensors.

More recent works leverage machine learning, but the majority of these approaches have considered ice motion prediction on a global scale for sea ice forecasting. Using satellite images, convolutional LSTM neural networks have been used to predict ice motion for several days in the future~\cite{predict_convLSTM, predict_attention_convLSTM}. In~\cite{predict_ML}, linear regression and CNN models are used to perform present-day Arctic sea ice motion prediction.

Few works address local ice floe motion with learning approaches. In~\cite{predict_river_ice}, a multi-step deep-learning based perception pipeline is proposed to track the motion and velocity of river ice, but no prediction is performed. In~\cite{predict_PINN_wave}, a physics-informed neural network incorporating attention mechanisms is proposed to predict ice floe longitudinal motion in response to waves. While this work addresses ice motion in the scale of our setting, it is not applicable to ASV planning as ship maneuvers are not considered. 
To the best of our knowledge, our work provides the first ice motion prediction model designed for real-time ASV local navigation in icy conditions. 
\section{Problem Formulation}
\label{sec:problem_formulation}
In this section we introduce the problem of ship navigation through an ice-covered waters with minimal collisions.

\subsection{Environment}
\label{sec:planning_formulation}

We define the water surface in which the ship navigates as a 2D surface $\mathcal{W} \subseteq \mathbb{R}^2$. We further restrict our navigation environment as a rectangular channel $\mathcal{C} \subset \mathcal{W}$ where the length is parallel to the y-axis, as shown in Fig.~\ref{fig:sim_fig_front} (left). Note that this restriction does not limit the applicability of planning, as longer and curved channels can be partitioned into rectangular sections \cite{rod2023ICRA}. The objective is to navigate the ship to pass a goal line $\mathcal{G} \subset \mathcal{C}$ with a constant y-value. 
We assume all ice pieces can be detected by some perception setup at any given time and we treat each ice piece as an obstacle. We group the obstacles into a set $W_{obs} = \{p_1, \ldots, p_m\}$ with $p \in W_{obs}$ denoting an individual obstacle and $m$ the total number of detected ice pieces. 

\subsection{Ice Motion Prediction}
\label{sec:occ_formulation}

To address the dynamic nature of ice floes in navigation, our objective is to predict ice field evolution as a result of ship maneuvers.
Given the obstacles $\mathcal{W}_{\text{obs}}$, we can compute a representation, denoted as $O$, that captures the state of the entire obstacle field. In this work, we use an occupancy map, but other representations could be used. 
Let $\mathcal{O}$, $\mathcal{S}$, and $\mathcal{A}$ be the state space of the obstacle field, ship state space, and ship action space, respectively, the prediction task is formulated as finding a function $f: \mathcal{O} \times \mathcal{S} \times \mathcal{A} \rightarrow \mathcal{O}$. Specifically, 

\begin{equation}
\label{eq:predict_generic}
    f(O, s, a) = O'
\end{equation}
takes as inputs the current state of the obstacle field $O \in \mathcal{O}$, the state of the ship $s \in \mathcal{S}$, and a ship action $a \in \mathcal{A}$ to predict the resulting state of the obstacle field $O' \in \mathcal{O}$ after the ship action is taken. In Sec.~\ref{sec:proposed_occ_prediction}, we detail the representation used for each of the arguments.

Note that the input $O$ and the prediction $O'$ being in the same space $\mathcal{O}$ allows making predictions on a sequence of actions. 
Specifically, given the current obstacle field state $O_1$ and a paired sequence of $n$ ship actions and ship state transitions $\{(s_1, a_1), (s_2, a_2), \ldots, (s_n, a_n)\}$, 
the resulting obstacle field state $O_n'$ after taking the $n$ actions can be predicted as 
\begin{equation}
\label{eq:sequential_generic}
    f(\ldots f(f(O_1, s_1, a_1), s_2, a_2) \ldots, s_n, a_n) = O_n'.
\end{equation}

We assume that the primary driving influence on ice motion in static waters is ship maneuvers. However, we acknowledge that other factors, such as wind and ship waves, also affect ice motion. The inclusion of these additional factors is left for future work.

\subsection{Path Planning}
Given a reference path $\Pi$ and let $d(\Pi)$ be the total path length of $\Pi$, we compute a cost function $u$ that penalizes both path length and collision as 
\begin{equation}
\label{eq:path_cost_generic}
    u(\Pi) = d(\Pi) + \alpha \cdot C(\Pi)
\end{equation}
where $C(\Pi) \geq 0$ is the collision cost function described in Section~\ref{sec:occ_diff_cost}, and $\alpha \geq 0$ is a tunable scaling parameter. Given a ship start position, a goal line $\mathcal{G}$, and a set of detected obstacles $\mathcal{W}_{obs}$, we seek a reference path $\Pi$ from start to $\mathcal{G}$ that minimizes $u$. 

\section{Ice Floe Motion Prediction}
In this section, we present our proposed deep learning pipeline to predict ice motions for ship planning. 
\label{sec:proposed_occ_prediction}

\subsection{Space Representations}
To account for the complex motions of ice floes, we leverage occupancy as a coarse representation of obstacle positions and orientations, shown in Fig.~\ref{fig:space_rep}a. The channel $\mathcal{C}$ is first discretized into a grid map $M$. Given the set of detected obstacles $W_{obs}$, we compute a global occupancy map $O_{\text{global}}$ of the same resolution as $M$ where each occupancy grid $O_{\text{global}}[x, y] \in [0, 1]$ represents the ratio between the obstacle-occupied area and the grid area at grid $M[x, y]$.  Note that our occupancy map differs from the typical form in the autonomous navigation literature, where occupancy maps are usually encoded as binary maps where each grid is either occupied or empty. By leveraging occupancy ratio, this coarse obstacle representation preserves the boundary and separation information of the obstacles. 

The state of the ship $s$ in Eq.~\ref{eq:predict_generic} is represented as the ship footprint to encode spatial features such as size and shape, shown in Fig.~\ref{fig:space_rep}c. Given the current pose $\eta$ of the ship, we define footprint $\mathcal{F(\eta)}$ to be the set of grid cells occupied by the ship body in map $M$.

We adopt the control set from~\cite{rod2023ICRA} as the action set of the ship, which is a set of motion primitives where each motion primitive is a short feasible path of the ship. Given a control set $\mathcal{P}$, the ship action $a$ in Eq.~\ref{eq:predict_generic} is effectively a motion primitive $\pi \in \mathcal{P}$, and is represented as the swath of the motion primitive. Specifically, we define the swath $\mathcal{T}(\pi)$ as the set of grid cells in $M$ that are swept by the ship footprint after executing the motion primitive $\pi$, as shown in Fig.~\ref{fig:space_rep}d. 

Note that the global occupancy map $O_{\text{global}}$, footprint $\mathcal{F}(\eta)$, and swath $\mathcal{T}(\pi)$ are all computed with the same resolution as the grid map $M$ such that the spatial relations between the obstacles, ship pose, and motion primitive are preserved in the representations. Further, given a control set $\mathcal{P}$, a path can be constructed by concatenating a sequence of motion primitives. The resulting occupancy after the entire path is traversed can then be predicted in a sequential manner, as outlined in Eq.~\ref{eq:sequential_generic}. An illustration of such sequential predictions for a concatenated path is presented in Fig.~\ref{fig:sim_fig_front} middle panels.

\begin{figure}[tbp]
    \centering
    \includegraphics[width=0.95\linewidth]{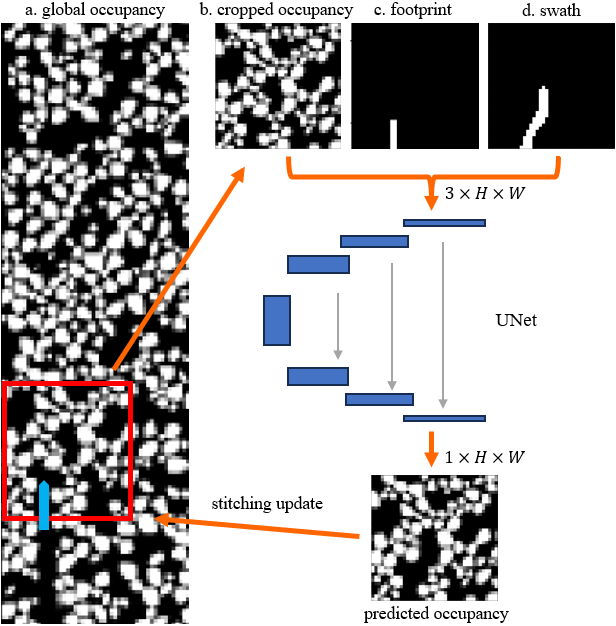}
    \caption{Ratio-based occupancy map is computed from the detected obstacles (a). Based on the ship's current pose (blue ship icon), a local occupancy observation (b) is cropped. Ship pose is encoded as footprint (c) and ship action is encoded as swath (d). U-Net is trained to predict the resulting occupancy. Predicted occupancy is stitched back into global occupancy.}
    \label{fig:space_rep}
    \vspace{-10pt}
\end{figure}

\subsection{Occupancy Prediction}

Observe that (Fig.~\ref{fig:space_rep}d and Fig.~\ref{fig:occ_diff_toy}) each motion primitive within the control set is a short path relatively local to the ship position with respect to the entire ice field. Hence a motion primitive is likely to only trigger ice motions within a local neighborhood. With this observation, we perform predictions within a locally cropped region $O \subseteq O_{\text{global}}$ with a fixed window size $H \times W$ from the global occupancy map based on the current ship position, as illustrated in Fig.~\ref{fig:space_rep}b. The footprint $\mathcal{F}(\eta)$ and swath $\mathcal{T}(\pi)$ are also spatially aligned with the cropped occupancy $O$.
After each prediction, we update the global occupancy map by simply stitching back the locally predicted occupancy. This local prediction approach allows our model to concentrate inference power to motion-rich areas. Special care is required to select an appropriate window size $H \times W$, as a window that is too small might not capture all ship-ice and ice-ice collisions, resulting in obstacles being pushed out of its boundary and hence inconsistency in global occupancy map updates. 

The neural network model input $X \in \mathbb{R}^{3 \times H \times W}$ is a 3-channel image where one channel is a local occupancy map observation $O \subseteq O_{\text{global}}$, one channel is the footprint $\mathcal{F}(\eta)$ and one channel is the swath $\mathcal{T}(\pi)$. Note that the footprint and swath channels are binary channels as footprint and swath are encoded as binary images. However, the occupancy channel is a gray-scale channel as each occupancy grid $O[x, y]$ is a value from $[0, 1]$.

Our model architecture is based on the U-Net~\cite{UNet_original}. In particular, our model has an encoder stack consisting of convolutional downsampling layers and a decoder stack with transposed convolutional upsampling layers and skip connections. Inspired by~\cite{implicitO}, residual convolution is used for the bottleneck to reduce model complexity. The model outputs a single channel of the predicted occupancy map with the same spatial dimension as the input occupancy map. 

\subsection{Model Training}
\subsubsection{Training Data}
We obtain our training data from the open-sourced ship-ice navigation simulator ({\footnotesize\url{https://github.com/rdesc/AUTO-IceNav}}) from~\cite{rod2023ICRA} in a self-supervised manner. The training set covers 20\%, 30\%, 40\%, and 50\% ice concentration environments and for each concentration, a total of 2,000,000 entries are collected. Each training entry is a tuple $(O, \mathcal{F}(\eta), \mathcal{T}(\pi), O')$ where the initial occupancy observation $O$, footprint $\mathcal{F}(\eta)$ and swath $\mathcal{T}(\pi)$ are the model input, and $O'$ is the ground truth resulting occupancy after the ship executing the motion primitive $\pi$ at pose $\eta$. To obtain a diverse set of ship-ice and ice-ice interaction samples, we deploy a random navigation policy for the ship during the data collection process. At each step, the ship randomly samples a motion primitive $\pi \in \mathcal{P}$ to execute. 

\subsubsection{Loss Functions}
We define the occupancy loss as a Huber loss between the predicted occupancy $\hat{O}'$ and the ground truth $O'$. 

To guide the learning process toward more physically-plausible predictions, we present a novel physics-based loss tailored to our occupancy formulation. The core assumption is that the total mass of the ice floes is conserved before and after the collisions from the ship executing a maneuver. 
Let $m_{p}$ represent the mass of an ice floe $p$, and 
let $W_{obs}^{i} = \{p^i_1, \ldots p^i_m\}$ and $W_{obs}^{f} = \{p^f_1, \ldots p^f_n\}$ be the sets of obstacle ice floes before and after ship executing a motion primitive, respectively. By conservation of mass, we have
\begin{equation}
\label{eq:conserve_mass}
    \sum_{p^i \in W_{obs}^{i}}m_{p^i} = \sum_{p^f \in W_{obs}^{f}}m_{p^f}.
\end{equation}
Similar to~\cite{rod2023ICRA}, we assume that the ice floes are 3D polygons with uniform density $\rho$ and thickness $d$. It follows that the mass of an ice floe is $m_p = \sigma_p \rho d$, where $\sigma_{p}$ denotes the surface area of an ice floe $p$. We rewrite Eq.~\ref{eq:conserve_mass} as
\begin{equation}
\label{eq:surface_area}
\begin{split}
   \sum_{p^i \in W_{obs}^{i}}\sigma_{p^i} \rho d = \sum_{p^f \in W_{obs}^{f}}\sigma_{p^f} \rho d \\
   \sum_{p^i \in W_{obs}^{i}}\sigma_{p^i} = \sum_{p^f \in W_{obs}^{f}}\sigma_{p^f}.
\end{split}
\end{equation}

Recall that each grid in our occupancy map is a ratio between the obstacle-occupied area and the grid area. 
Let $s_g$ denote the grid size of the occupancy map, and let $o_{p}$ be the total occupancy value the ice floe $p$ contributes to the occupancy map, it follows that $\sigma_{p} = s_g o_{p}$. Further, note that the sum of all ice floes' occupancy $o_{p}$ is simply the sum of occupancy values in the occupancy map, namely $\sum_{p^i \in W_{obs}^i}o_{p^i} = \textup{Sum}(O)$ and $\sum_{p^f \in W_{obs}^f}o_{p^f} = \textup{Sum}(O')$.
With these observations, Eq.~\ref{eq:conserve_mass} can be rewritten as 
\begin{equation}
\label{eq:occ_area}
\begin{split}
    \sum_{p^i \in W_{obs}^{i}}o_{p^i} s_g = \sum_{p^f \in W_{obs}^{f}}o_{p^f} s_g \\
    \sum_{p^i \in W_{obs}^{i}}o_{p^i} = \sum_{p^f \in W_{obs}^{f}}o_{p^f} = \\
    \textup{Sum}(O) = \textup{Sum}(O').
\end{split}
\end{equation}
In other words, the sum of all cell values from the occupancy map before and after executing a motion primitive is conserved. Hence, given a training set entry $(O, \mathcal{F}(\eta), \mathcal{T}(\pi), O')$ and a prediction $\hat{O}'$, we define a conservation loss as $\text{MSE}(\textup{Sum}(O), \textup{Sum}(\hat{O}'))$. The final loss is constructed as
\begin{equation}
\label{eq:training_loss_final}
    \mathcal{L} = \text{Huber}(O', \hat{O}') + \lambda \textup{MSE}(\textup{Sum}(O), \textup{Sum}(\hat{O}')),
\end{equation}
which is a linear combination of the occupancy loss and conservation loss scaled by $\lambda$. In practice, the inclusion of the conservation loss prevents the neural network predictor from expanding or shrinking the obstacles in the occupancy map predictions, preserving obstacle shapes and sizes.

\section{Path Planning Framework}
\label{sec:proposed_planning}

We integrate a lattice-based planner with the trained neural network predictor via a simple yet empirically effective cost function based on occupancy map difference. 
We reduce the number of network inferences by caching intermediate prediction results. 

\subsection{Lattice Planning with Motion Primitives}
We adopt the lattice-based planner from~\cite{rod2023ICRA}. The robot state space is discretized into a state lattice and a pre-computed set of motion primitives, called a control set $\mathcal{P}$, is used to repeatedly sample feasible motions. Each motion primitive $\pi \in \mathcal{P}$ is a path that connects two lattice states. 
The total planned path $\Pi = \{\pi_1, \pi_2, \ldots, \pi_n\}$ is then a concatenated sequence of motion primitives $\pi_i$ where $i = 1, \ldots, n$. 
Finally, the planning problem is reduced to searching for a sequence of motion primitives that minimizes the cost function at Eq.~\ref{eq:path_cost_generic} from the start state to the goal line.

\subsection{Collision Cost with Occupancy}
\label{sec:occ_diff_cost}
To leverage ice motion prediction, our key observation is that the change in ice floe occupancy map is strongly correlated with our physical quantities of interests, such as the impulse and kinetic energy loss from ship-ice collisions. We define an occupancy difference function $\textup{diff}: \mathcal{O} \times \mathcal{O} \rightarrow \mathbb{R}$ to compute a scalar difference between two occupancy maps. Let $O$ and $O'$ be the ice floe occupancy maps before and after the ship executes a motion primitive, respectively. 

We begin by presenting our empirical analysis on how the occupancy change, denoted as $\textup{diff}(O, O')$, that results from a ship motion primitive relates to the ship's kinetic energy loss, collision impulse and approximated work done ($W_{\text{approx}}$ described in Sec.~\ref{sec:simulation-setup}) during the execution of the motion primitive. For this analysis, we use occupancy maps that match the input/output dimensions of our learned predictor.

We examined three different methods to measure occupancy differences: mean-squared error (MSE), negated structural similarity (-SSIM), and earth-mover distance (EMD).
Experiments of ship-ice collisions are generated from 100 trials of random ship navigations in each of the 20\%, 30\%, 40\%, and 50\% concentration environments. These experiments provide collision entries that contain the occupancy changes measure by the three difference methods and the physical quantities of interests after the ship executing a motion primitive. 
The correlations between the occupancy changes and the collision metrics are presented in Table~\ref{table:occ_correlation}. As shown, occupancy changes from all difference-measurement methods are highly correlated with the kinetic energy loss, impulse and approximated work. 
\begin{table}
\centering
\begin{tabular}{lccc}
    \toprule
     & KE Loss & Impulse & $W_{\text{approx}}$ \\
     \midrule
    \textit{MSE} & 0.77 & 0.95 & 0.96 \\
    \textit{-SSIM} & 0.70 & 0.88 & 0.90  \\
    \textit{EMD} & 0.66 & 0.88 & 0.89 \\
\bottomrule
\end{tabular}
\caption{Occupancy change correlations}
\label{table:occ_correlation}
\vspace{-20pt}
\end{table}
This strong correlation is intuitive because the occupancy difference correlates with ice field changes. A larger ice field change indicates a greater displacement of ice floes by the ship maneuvers, which results in heavier ship-ice collisions and greater ship navigation efforts. 

Given this observation, we formulate our collision cost based on occupancy difference, assuming that minimizing occupancy changes minimizes collisions. Specifically, let $\pi \in \mathcal{P}$ be a motion primitive, the collision cost $c_{\pi}$ of taking this motion primitive is computed as 
\begin{equation}
\label{eq:collision_cost}
    c_{\pi} = \textup{diff}(O, O'),
\end{equation}
where $O$ and $O'$ are the ice floe occupancy maps before and after the ship executes the motion primitive $\pi$, respectively. For the rest of this paper, we adopt MSE as the method for computing occupancy difference, given its highest empirical correlation with our collision metrics and its computational efficiency. Fig.~\ref{fig:occ_diff_toy} shows an illustrative example of this cost. 

\setlength{\belowcaptionskip}{-10pt}
\begin{figure}[tbp]
    \centering
    \includegraphics[width=\linewidth]{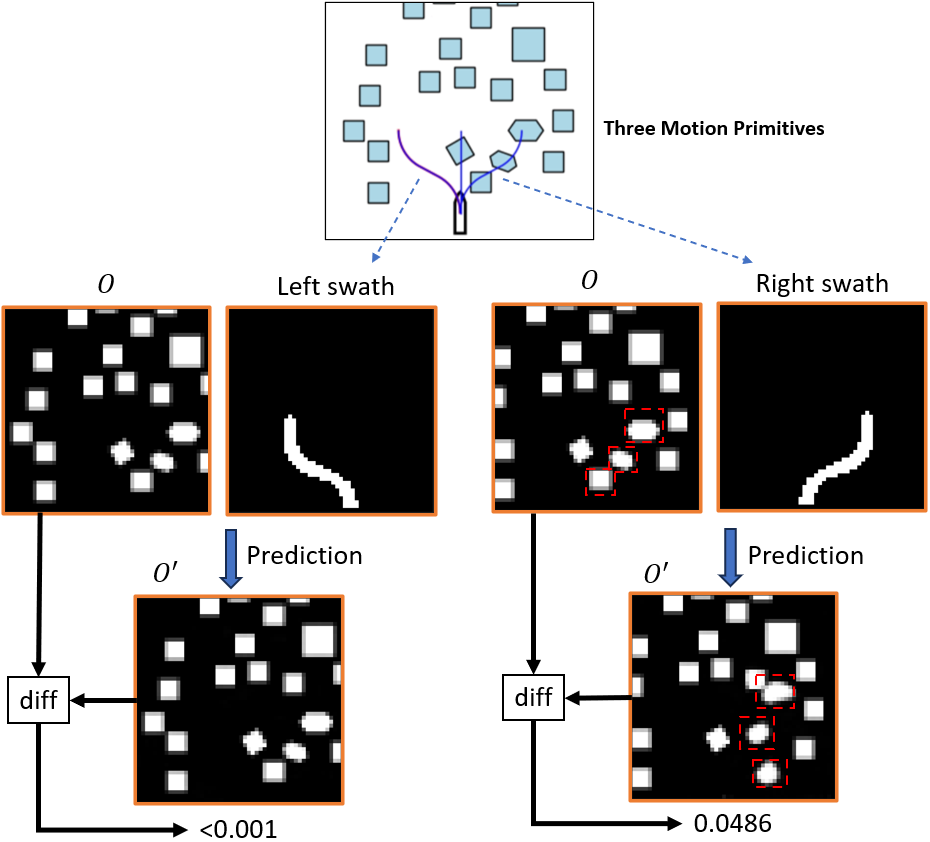}
    \caption{A toy example showing occupancy differences. The motion primitives are simplified for illustration purposes. The red dotted boxes highlight the predicted motions of the obstacles. Here left swath is optimal.}
    \label{fig:occ_diff_toy}
\end{figure}
\setlength{\belowcaptionskip}{-10ex}

Finally, given the collision cost from Eq.~\ref{eq:collision_cost} and a concatenated total path $\Pi = \{\pi_1, \pi_2, \ldots, \pi_n\}$ from a sequence of motion primitives $\pi_i$ where $i = 1, \ldots,n$, we define the collision cost function $C$ in Eq.~\ref{eq:path_cost_generic} for the entire path as 
\begin{equation}
\label{eq:path_cost_final}
    C(\Pi) = \sum_{\pi \in \Pi} \text{diff}(O_{\pi}, O'_{\pi}),
\end{equation}
where $O_{\pi}$ and $O'_{\pi}$ is the ice floe occupancy before and after the ship executing motion primitive $\pi$, respectively.

\subsection{A* Search with Occupancy Estimation}
\label{sec:a_star_search}

Presented in Alg.~\ref{alg:a_star_occ_diff}, our lattice-based planner is based on the A* search algorithm~\cite{a_star} and integrates the predictive component to be obstacle-motion-aware during planning. The algorithm takes as input the current $2D$ pose of the ship as the start node $n_s = \langle x_s, y_s, \theta_s \rangle$, the current global occupancy observation $O^{\textup{global}}_s$, and a goal line $y_{\text{goal}}$. 

During each iteration, the current node $n$ is expanded to its neighbor $n'$ via a feasible motion primitive $\pi$, denoted by the overloaded $+$ operator at Line 11.
A local occupancy $O$ is cropped from the global occupancy $O^{\textup{global}}$ based on the ship pose at node $n$ (Line 13) and is then used for occupancy prediction (Line 14).
The edge cost of $\pi$ is then computed as a linear combination of the occupancy-difference collision cost discussed in~\ref{sec:occ_diff_cost} scaled by $\alpha$ and the motion primitive length $d(\pi)$ (Line 15). The predicted local occupancy $O'$ is stitched back for an updated global occupancy map $O'^{\textup{global}}$ (Line 16), which is then cached if the neighbor $n'$ is visited for the first time or revisited with a better path (Lines 22). In a future iteration when the node $n'$ is being expanded, the global occupancy at $n'$ predicted previously is retrieved for further predictions (Line 9). 
Additionally, by leveraging a GPU in practice, the network inference can be further optimized by performing batched predictions by constructing the inferences from node $n$ to all its neighbors $n'$ as a batch.

\SetKwComment{Comment}{// }{}
\newcommand{\mycommfont}[1]{\small{#1}}
\SetCommentSty{mycommfont}
\SetKw{Continue}{continue}

\begin{algorithm}[htbp]
    \small
    \DontPrintSemicolon
    \caption{A* Search with Occupancy Prediction} 
    \label{alg:a_star_occ_diff}
    \KwIn{$n_s$, $O^{\text{global}}_s$, $y_{\text{goal}}$}
    
    $\text{OPEN} \leftarrow \text{min priority queue containing } n_s$\;
    $\text{CLOSED} \leftarrow \text{empty set}$\;
    $g \leftarrow \text{dict that stores cost-so-far}$\;
    $\textup{occMemo}[n_s] \leftarrow O^{global}_s$   \Comment*[r]{cache first global occupancy}
    \While{\textup{OPEN} \textup{not empty}}{
        $n \leftarrow \text{OPEN.extractMin()}$\;
        \lIf{$n$ \textup{reaches} $y_{\text{goal}}$} {
            \textup{reconstruct and \Return path}
        }
        $\text{CLOSED.add(n)}$\;
        $O^{\text{global}} \leftarrow \text{occMemo[n]}$  \Comment*[r]{retrieve cached prediction}
        \For{all $\pi$ in $\mathcal{P}$} {
            $n' \leftarrow n + \pi$ \Comment*[r]{get neighbor}

            \lIf{$n' \in $ \textup{CLOSED}} {
                \Continue
            }
            $O \leftarrow \textup{crop}(O^{\text{global}}, n)$      \Comment*[r]{crop local occupancy}
            $O' \leftarrow f(O, \mathcal{F}(n), \mathcal{T}(\pi))$  \Comment*[r]{prediction, Sec.~\ref{sec:proposed_occ_prediction}}
            
            $\text{cost} \leftarrow g(n) + \alpha \cdot \textup{diff}(O, O') + d(\pi)$\;

            $O'^{\textup{global}} \leftarrow \textup{stitch}(O^{\textup{global}}, O')$  \Comment*[r]{stitch prediction}

            \If{$n' \notin $ \textup{OPEN} \textup{\textbf{or}} $\textup{cost} < g(n')$}{                          
                $g(n') \leftarrow \text{cost}$\;
                $F \leftarrow g(n') + h(n')$     \Comment*[r]{$h$ heuristic from~\cite{rod2023ICRA}}
                $\text{OPEN.enq}(n', F)$        \Comment*[r]{enqueue or update $n'$}
                \text{Update $n'$ parent pointer}\;
                $\text{occMemo}[n'] \leftarrow O'^{\textup{global}}$  \Comment*[r]{cache prediction}
            }
        }
    }
    \Return False
\end{algorithm}

\subsection{Algorithm Analysis}
Here we analyze Alg.~\ref{alg:a_star_occ_diff} runtime and performance. 

\subsubsection{Runtime}
 While Alg.~\ref{alg:a_star_occ_diff} has the same structure as $A^*$ or Dijkstra's algorithm, given that the predictor $f$ is a deep neural network, the bottleneck operation is the prediction in Line 14, rather than the priority queue operation in typical cases. Note that in Alg.~\ref{alg:a_star_occ_diff}, each node is  expanded at most once and is then added to the closed set $\textup{CLOSED}$ (Line 8). This implies that Alg.~\ref{alg:a_star_occ_diff} runtime scales linearly with the number of edges, or motion primitives (Line 10), explored during the search. Let $|E|$ be the number of edges in the expanded search graph from Alg.~\ref{alg:a_star_occ_diff}, it follows that the runtime is $O(|E|P)$, where $P$ is the time for each forward pass of the predictor.

\subsection{Performance}
We note that by considering obstacle motions, the path finding problem lacks optimal substructures. Namely, in an arbitrary optimal path $n_s \leadsto n \leadsto n_{\textup{goal}}$ from start $n_s$ to the goal $n_{\textup{goal}}$ that passes through node $n$, the sub-path $n_s \leadsto n$ is not necessarily the optimal path from $n_s$ to $n$. This could be the cases where, by following the sub-optimal path $n_s \leadsto n$ to $n$, the ship pushes the ice floes into a desirable configuration, resulting in the entire path $n_s \leadsto n \leadsto n_{\textup{goal}}$ being optimal.
Further, optimally solving the navigation among movable obstacle problem is generally considered as NP-hard~\cite{namo_np_hard}. Consequently, Alg.~\ref{alg:a_star_occ_diff} lacks theoretical optimality guarantee.
Hence, we present the following constant-factor guarantee to upper-bound the performance of Alg.~\ref{alg:a_star_occ_diff}.
\begin{theorem}
\label{theorem:worst-path}
Given a problem instance $I = (n_s, O_s, y_{\text{goal}})$, let $\textup{OPT}(I)$ denote the optimal path and let $\textup{Alg}(I)$ denote the path returned by Alg.~\ref{alg:a_star_occ_diff} assuming perfect occupancy predictions. Let $l_{\min}$ be the distance of the shortest action primitive, the cost of $\textup{Alg}(I)$ is upper-bounded as 
\begin{align}
    u(\textup{Alg}(I)) \leq (1 + \frac{\alpha}{l_{\min}}) \cdot u(\textup{OPT}(I))
\label{eq:worst-case-bound-result}
\end{align}
where $u$ and $\alpha$ are described in Eq.~\ref{eq:path_cost_generic}. 
\end{theorem}
We present the proof of Theorem.~\ref{theorem:worst-path} in the next section. We acknowledge that Theorem.~\ref{theorem:worst-path} may not hold under imperfect predictions, and we leave the incorporation of prediction error uncertainty to future work.

\section{Algorithm Performance Bound}
\label{sec:performance_bound}
Recall that given an edge $e = (n_1, n_2)$ that connects two nodes $n_1$ and $n_2$, the cost of the edge $u(e)$ is computed as 
\begin{equation}
\label{eq:edge-cost}
    u(e) = d(e) + \alpha \cdot \textup{MSE}(O_1, O_2),
\end{equation}
where $O_1$ and $O_2$ are the occupancy maps at nodes $n_1$ and $n_2$, respectively (Alg.~\ref{alg:a_star_occ_diff}, Line 15). This is a linear combination of the edge length $d(e)$ and the occupancy difference from taking the edge. For simplicity we use $u(\cdot)$ to denote both the cost of a path and an edge. By noting that each cell value of an occupancy map $O$ is a value from $[0, 1]$, we present the following observation.

\begin{observation}[Maximal and Minimal Edge Weights]
For each edge $e$, the edge cost is bounded as $d(e) \leq u(e) \leq d(e) + \alpha$.
\label{obs:edge-weights} 
\end{observation}

Let $I = (n_s, O_s, y_{\text{goal}})$ denote an arbitrary problem instance and let $\textup{OPT}(I) = n_0 \rightarrow n_1 \rightarrow \ldots \rightarrow n_{k-1} \rightarrow n_{k}$ be the optimal path with $k \geq 1$ edges, where $n_0 = n_s$ and $n_k = n_{\textup{goal}}$. We let $e_i$ with $i \in \{1, \ldots, k\}$ denote the directed edge on the optimal path from node $n_{i-1}$ to node $n_i$. 

Recall that in our problem setup the goal is defined to be a line $y_{\textup{goal}}$. So here $n_k$ is an arbitrary goal node $n_{\textup{goal}}$ whose y-value is greater than $y_{\textup{goal}}$, and the multiplicity of the goal nodes does not affect our proof. We now prove the following Lemma.

\begin{lemma}
\label{lemma:upper-bound}
For all nodes $n_i$ with $i = \{1, \ldots, k\}$ on the optimal path, after running Alg.~\ref{alg:a_star_occ_diff}, the cost-so-far for $n_i$ found by Alg.~\ref{alg:a_star_occ_diff}, denoted as $g(n_i)$ (Alg.~\ref{alg:a_star_occ_diff} Line 18), is upper-bounded by 

\begin{equation}
\label{eq:alg-upper-bound-0}
\begin{split}
    & g(n_i) \leq d(e_1) + \ldots + d(e_i) + \alpha \cdot i.
\end{split}
\end{equation}
\end{lemma}

\begin{proof}
We prove Lemma~\ref{lemma:upper-bound} by induction on the node $n_i$ with $i \in \{1, \ldots, k\}$. 

\textbf{Base Case $n_i = n_1$: } There are two cases:

\begin{enumerate}
    \item After the execution of Alg.~\ref{alg:a_star_occ_diff}, $n_1$'s parent pointer points to $n_0$. This implies that $g(n_1) = u(e_1) \leq d(e_1) + \alpha$ by Observation~\ref{obs:edge-weights}.
    \item After the execution of Alg.~\ref{alg:a_star_occ_diff}, $n_1$'s parent pointer does not point to $n_0$. This implies Alg.~\ref{alg:a_star_occ_diff} finds a cheaper path to $n_1$ than following the edge $e_1$. It follows that $g(n_1) \leq u(e_1) \leq d(e_1) + \alpha$. Thus, the base case holds.
\end{enumerate}

\textbf{Inductive Hypothesis: } Lemma~\ref{lemma:upper-bound} holds for all nodes $n \in \{n_1, \ldots, n_{i-1}\}$.

\textbf{Inductive Step: } By Inductive Hypothesis, it follows that after the execution of Alg.~\ref{alg:a_star_occ_diff}, we have
\begin{equation}
\label{eq:i-1-cost}
    g(n_{i-1}) \leq d(e_1) + \ldots + d(e_{i-1}) + \alpha \cdot (i-1).
\end{equation}
We again proceed with the same case analysis. 
\begin{enumerate}
    \item After the execution of Alg.~\ref{alg:a_star_occ_diff}, $n_i$'s parent pointer points to $n_{i-1}$. This implies that $g(n_i) = g(n_{i-1}) + u(e_i)$. By Eq.~\ref{eq:i-1-cost} and Observation~\ref{obs:edge-weights}, we have
    \begin{equation}
    \label{eq:inductive-step-1}
    \begin{split}
        & g(n_i) = g(n_{i-1}) + u(e_i) \\
        & \leq d(e_1) + \ldots + d(e_{i-1}) + \alpha \cdot (i-1) + d(e_i) + \alpha \\
        & = d(e_1) + \ldots + d(e_i) + \alpha \cdot i.
    \end{split}
    \end{equation}
    \item After the execution of Alg.~\ref{alg:a_star_occ_diff}, $n_i$'s parent pointer does not point to $n_{i-1}$. This implies Alg.~\ref{alg:a_star_occ_diff} finds a cheaper path to $n_i$ than following the edge $e_i$ from node $n_{i-1}$. Consequently, $g(n_i) \leq g(n_{i-1}) + u(e_i)$. Applying the same reasoning from Eq.~\ref{eq:inductive-step-1}, we have 
    \begin{equation}
    \label{eq:inductive-step-2}
    \begin{split}
        & g(n_i)\leq d(e_1) + \ldots + d(e_i) + \alpha \cdot i.
    \end{split}
    \end{equation}
\end{enumerate}
Thus, by induction, Lemma~\ref{lemma:upper-bound} holds.
\end{proof}

With Lemma~\ref{lemma:upper-bound} established, note that $n_k = n_{\textup{goal}}$. This implies that after its execution, Alg.~\ref{alg:a_star_occ_diff} is guaranteed to find a path to $n_{\textup{goal}}$ that is cheaper than $d(e_1) + \ldots + d(e_k) + \alpha \cdot k$. Let $\textup{Alg}(I)$ denote the path returned by Alg.~\ref{alg:a_star_occ_diff} on instance $I$, it follows that 
\begin{equation}
\label{eq:alg-upper-bound-1}
\begin{split}
    & u(\textup{Alg(I)}) \leq d(e_1) + \ldots + d(e_k) + \alpha \cdot k \\
    & = d(\textup{OPT}(I)) + \alpha \cdot k.
\end{split}
\end{equation}

Recall that $k$ is the number of edges in the optimal path $\textup{OPT}(I)$, and each edge in Alg.~\ref{alg:a_star_occ_diff} is an action primitive $\pi$. 
In this work, all distances $d(\cdot)$ are measure in meters. Let $l_{\min}$ denotes the distance of the shortest action primitive in meters. 
This implies that $k \leq \frac{d(\textup{OPT}(I))}{l_{\min}}$. By expanding $u(\cdot)$ using its definition from Eq.~\ref{eq:path_cost_generic}, it follows that
\begin{equation}
\label{eq:alg-upper-bound-2}
\begin{split}
    & d(\textup{OPT}(I)) + \alpha \cdot k\\
    & \leq d(\textup{OPT}(I)) + \alpha \cdot k + \alpha \cdot C(\textup{OPT}(I))\\
    & \leq d(\textup{OPT}(I)) + \frac{\alpha}{l_{\min}} \cdot d(\textup{OPT}(I)) + \alpha \cdot C(\textup{OPT}(I))\\
    & = (1 + \frac{\alpha}{l_{\min}}) \cdot d(\textup{OPT}(I)) + \alpha \cdot C(\textup{OPT}(I))\\
    & \leq (1 + \frac{\alpha}{l_{\min}}) \cdot (d(\textup{OPT}(I) + \alpha \cdot C(\textup{OPT}(I)))\\
    & = (1 + \frac{\alpha}{l_{\min}}) \cdot u(\textup{OPT}(I)).
\end{split}
\end{equation}

It follows that Theorem~\ref{theorem:worst-path} holds.
\section{Evaluation}
\label{sec:evaluation}
We evaluate the performance of the proposed planner both in simulation and in a physical dry-land ship-ice navigation testbed. We compare our planner (\emph{predictive}) against three baselines. The first baseline (\emph{straight}) is a naive planner that returns a straight path from the ship's position to the goal. The second baseline (\emph{skeleton}) adopts the shortest open-water planner proposed in~\cite{skeleton_planner}. The third baseline (\emph{lattice}) is the lattice planner proposed in~\cite{rod2023ICRA}. Each trial involves navigating the ship across an ice channel toward a goal line $\mathcal{G}$ using a specified planner. Codes and demos are available at {\footnotesize \url{https://github.com/IvanIZ/predictive-asv-planner}}.

\subsection{Simulation Setup}
\label{sec:simulation-setup}
We conduct evaluations on the autonomous ship-ice navigation simulator from~\cite{rod2023ICRA}. The setup is consistent with the experimental platform at the National Research Council Canada ice tank in St. John’s, NL, featuring a 1:45 scale model vessel. 
Tests are conducted at 20\%, 30\%, 40\%, and 50\% concentrations, with 200 randomly generated environments per concentration. Fig.~\ref{fig:sim_fig_front} (left) shows a 40\% concentration ice field. We evaluate all four planners on the same set of trials, 
giving a total of  \(4 \, \text{concentrations} \, \times 200 \, \text{trials} \, \times 4 \, \text{planners} \, = 3200 \, \text{experiments}\). 

Planners are evaluated with four metrics - travel distance, a running total kinetic energy loss by the ship, impulse due to ship-ice collisions, and an approximate of work done by the ship due to ice pushing. The total kinetic energy loss and impulse are computed from the simulator, and the approximated work $W_{\text{approx}}$ is computed as an aggregation of products of the ice mass $m_{p_i}$ with the arc length $s_{p_i}$ taken by the ice floe $p_i \in W_{obs}$ as $W_{\text{approx}}= \sum_{p_i \in W_{obs}} m_{p_i} s_{p_i}$.
While $W_{\text{approx}}$ does not directly compute work, we found that it correlates with work and provides a direct assessment of the extent of ice field changes due to ship maneuvers.

\subsection{Simulation Results}
The evaluation results from simulation are presented in Fig.~\ref{fig:eval_result_sim}. As shown, the \emph{predictive} planner gives the best performance across all collision metrics in all concentrations.

In general, the \emph{predictive} planner and the \emph{lattice} are the most competitive planners in collision minimization, especially in 20\% and 30\% concentrations. This could potentially be explained by the fact that both \emph{predictive} and \emph{lattice} planners are designed to handle ship-ice navigation scenarios where a collision-free path does not exist. 

Observe that while the collision minimization performances of \emph{predictive} and \emph{lattice} planner are comparably competitive in low ice concentrations (20\% and 30\%), the performance gain from \emph{skeleton} and \emph{lattice} planners over the naive \emph{straight} planner diminishes significantly as concentration increases (40\% and 50\%). Since \emph{skeleton} and \emph{lattice} planner do not account for ice floe motion during planning, the increasingly rich and complex ice motion in higher concentrations renders the planning strategies less effective. However, the \emph{predictive} planner maintains a robust performance across all concentrations. 

Lastly, note that \emph{predictive} planner outperforms \emph{lattice} while maintaining comparable distances. This suggests that the improved collision avoidance of the \emph{predictive} planner is not simply due to compromising distance by taking longer detours, but stems from a more informed reasoning process enabled by the additional knowledge about ice motion.

\begin{figure}[htbp]
    \centering
    \begin{minipage}[b]{0.48\linewidth}
        \centering
        \includegraphics[width=\linewidth]{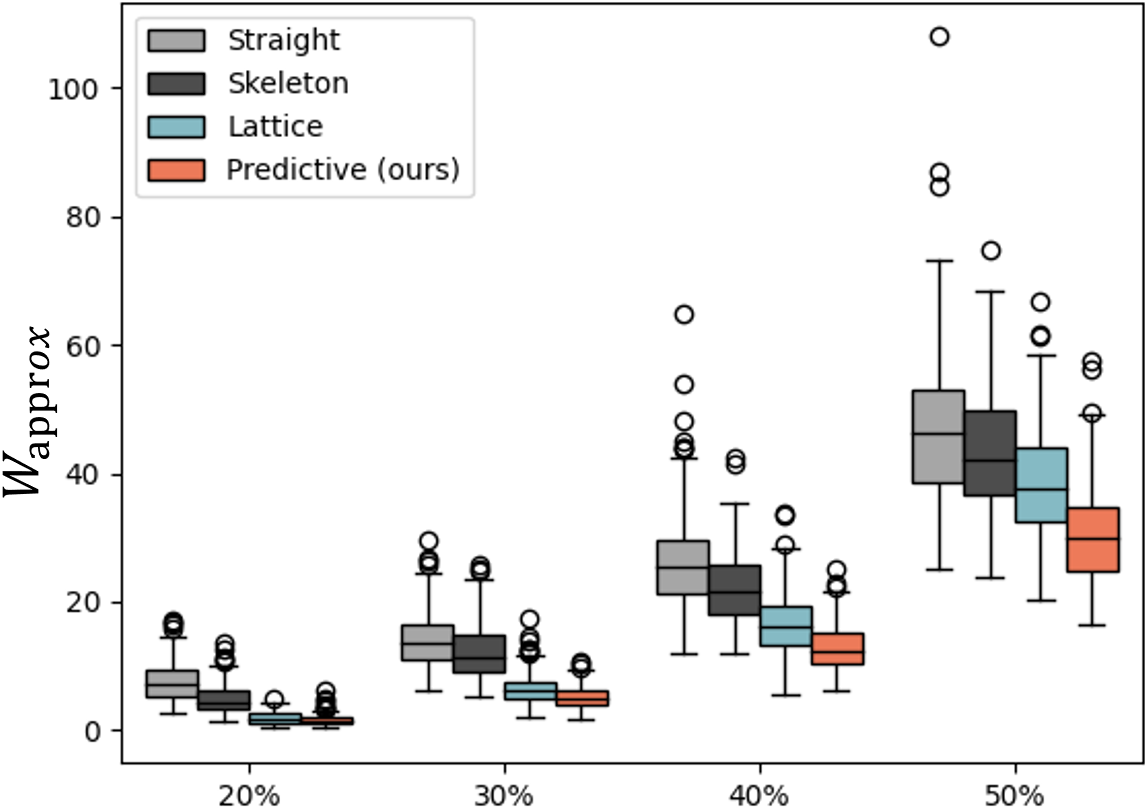}
    \end{minipage}
    \hspace{0.01\linewidth} 
    \begin{minipage}[b]{0.48\linewidth}
        \centering
        \includegraphics[width=\linewidth]{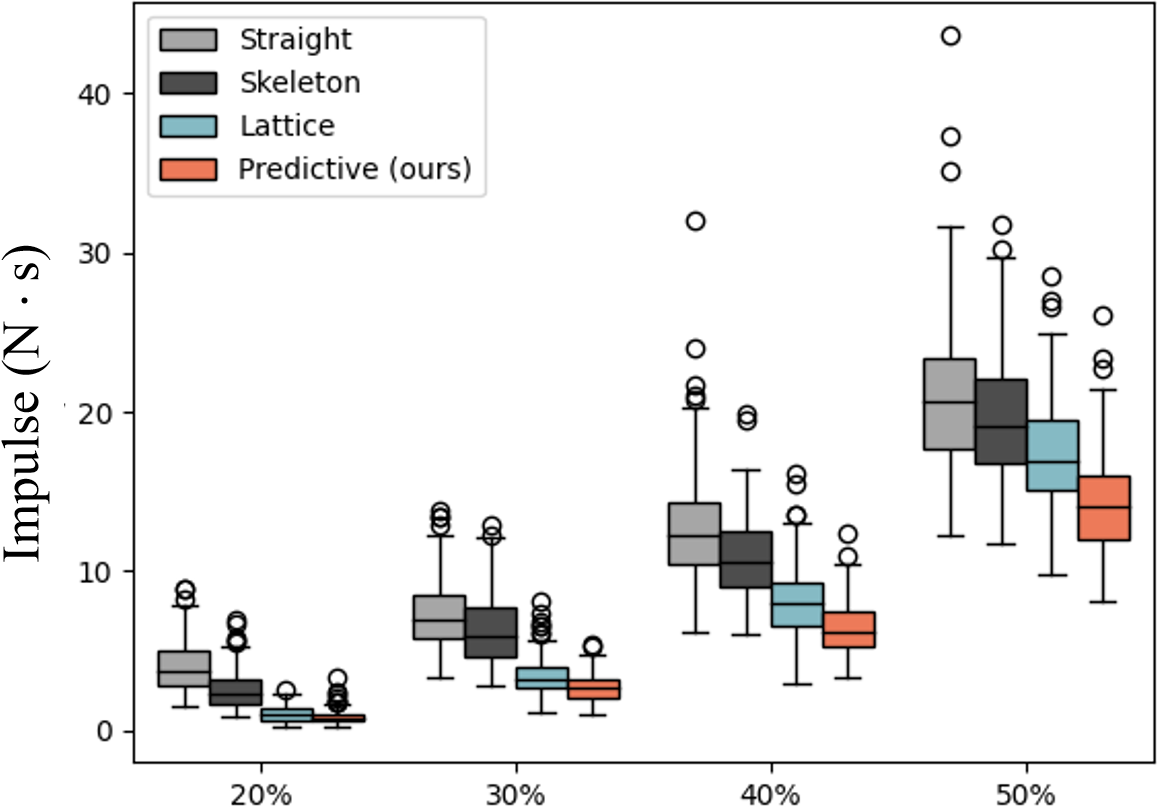}
    \end{minipage}
    
    \vspace{1mm}
    
    \begin{minipage}[b]{0.48\linewidth}
        \centering
        \includegraphics[width=\linewidth]{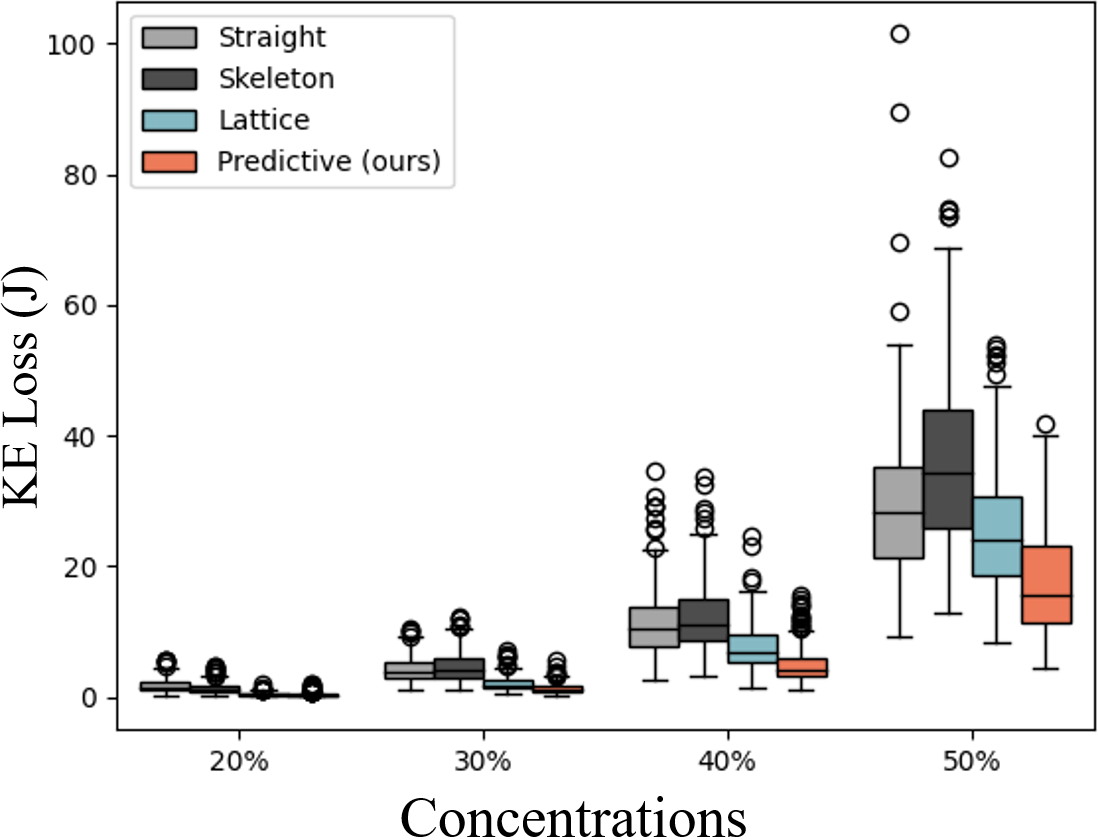}
    \end{minipage}
    \hspace{0.01\linewidth}
    \begin{minipage}[b]{0.48\linewidth}
        \centering
        \includegraphics[width=\linewidth]{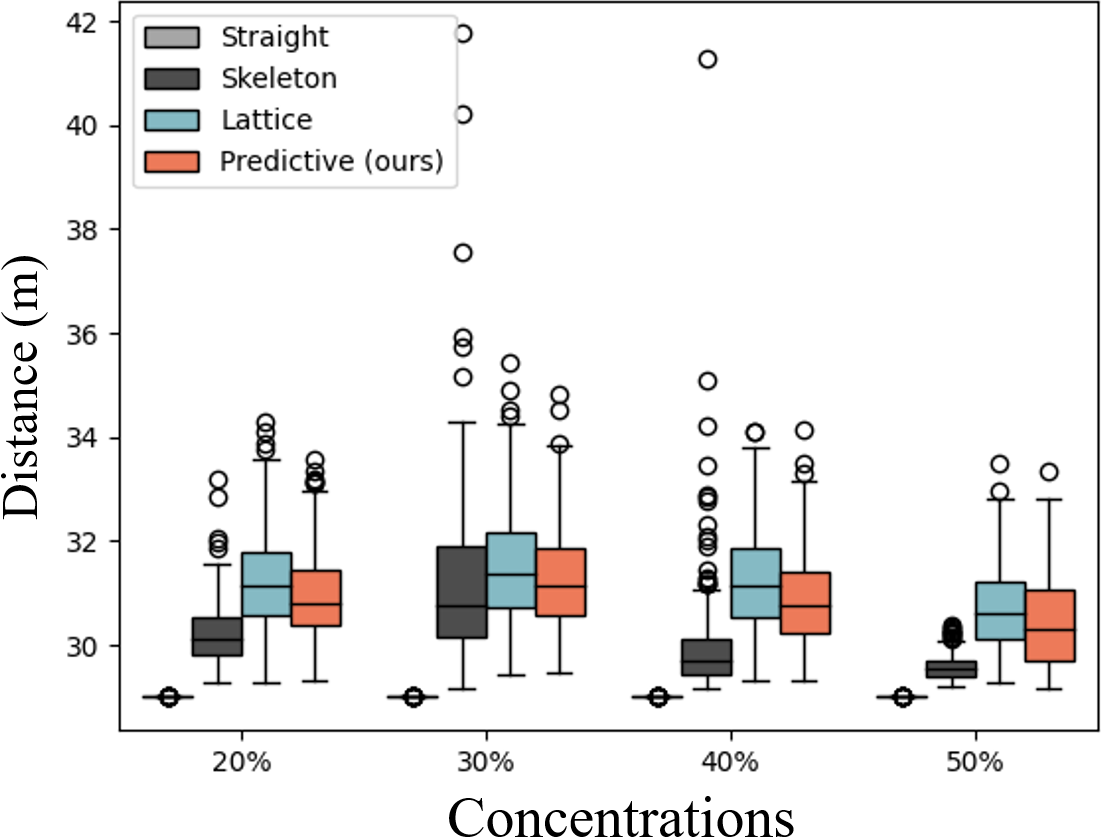}
    \end{minipage}
    \caption{Evaluation results from simulation across four concentrations}
\label{fig:eval_result_sim}
\vspace{-10pt}
\end{figure}

\subsection{Physical Testbed Setup}
We further validate the proposed planner in a physical testbed that simulates autonomous ship-ice navigation at the University of Waterloo Autonomous System Laboratory.  The testbed features a \(2.8 \, \text{m} \, \times 1.6 \, \text{m}\) navigation environment, shown in Fig.~\ref{fig:sim_fig_front} (right). A TurtleBot3 Burger is used as a model vessel, and is attached with V-shape hanger as the ship bow. Ice obstacles are made from randomly shaped foam boards. The vision system is an overhead camera that tracks the robot pose at 10\,Hz and detects obstacles at 3\,Hz. The robot navigates at a constant nominal velocity of \(0.03 \, \text{m/s}\), with a PID controller tracking the reference path at 10\,Hz.

\subsection{Physical Testbed Evaluation}
We perform testbed evaluations in 20\%, 30\%, and 40\% concentration environments. For each concentration, 40 trials are performed for each planner, 
giving a total of \(3 \, \text{concentrations} \, \times 40 \, \text{trials} \, \times 4 \, \text{planners} \, = 480 \, \text{experiments}\). 
$W_{\text{approx}}$ and travel distance are used as metrics to quantify collision minimization and travel efficiency, respectively. 
Demonstrations of the testbed setup and experiments can be found in the supplemental video.

Fig.~\ref{fig:eval_result_testbed} presents the evaluation results from the testbed, showing similar trends. Notably, the \emph{predictive} planner achieves the lowest $W_{\textup{approx}}$ across all three concentrations while maintaining comparable travel distances in the 20\% and 30\% cases. However, a trend of increased travel distance from \emph{predictive} is observed at the 40\% concentration, likely due to greater errors in ice motion prediction.

We acknowledge the lab-to-real gap in our testbed, particularly the absence of factors such as ocean current and fluid dynamics. However, the testbed effectively captures core aspects of autonomous navigation in ice-covered waters, focusing on scenarios where a collision-free path is unattainable and where ship-ice and ice-ice collisions are prominent.

\begin{figure}[htbp]
    \centering
    \begin{minipage}[b]{0.48\linewidth}
        \centering
        \includegraphics[width=\linewidth]{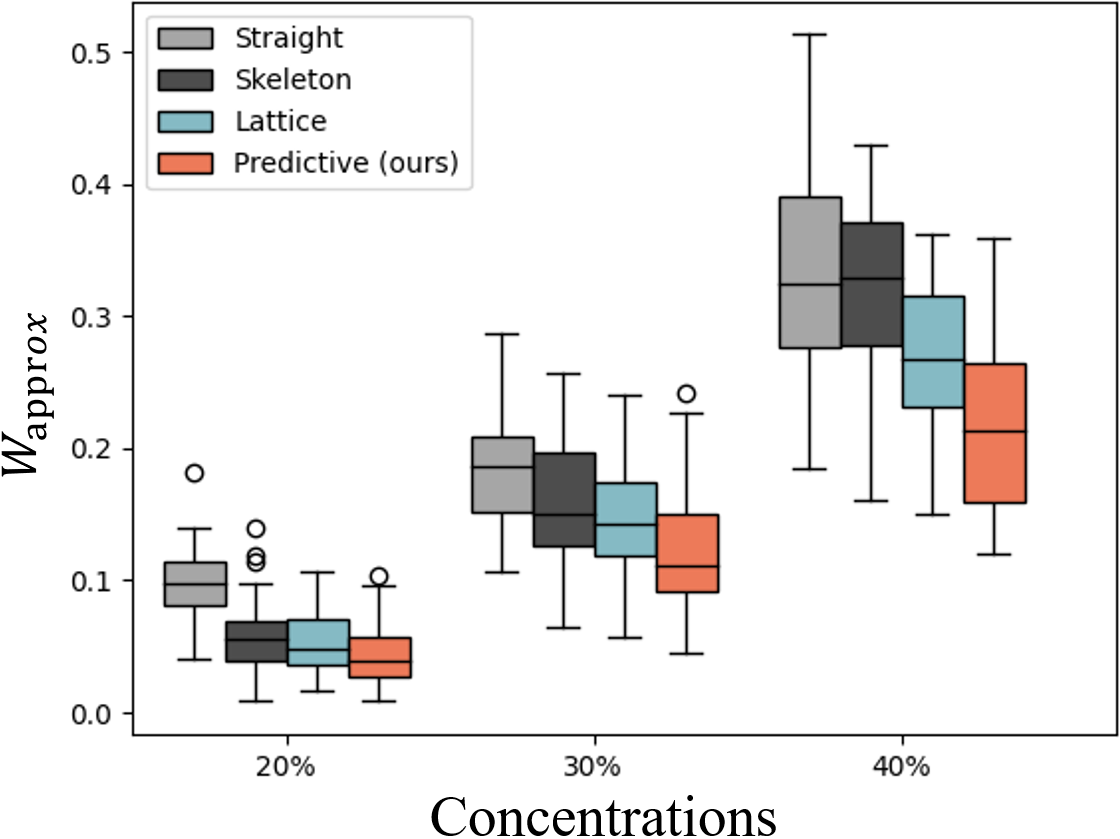}
    \end{minipage}
    \hspace{0.01\linewidth} 
    \begin{minipage}[b]{0.48\linewidth}
        \centering
        \includegraphics[width=\linewidth]{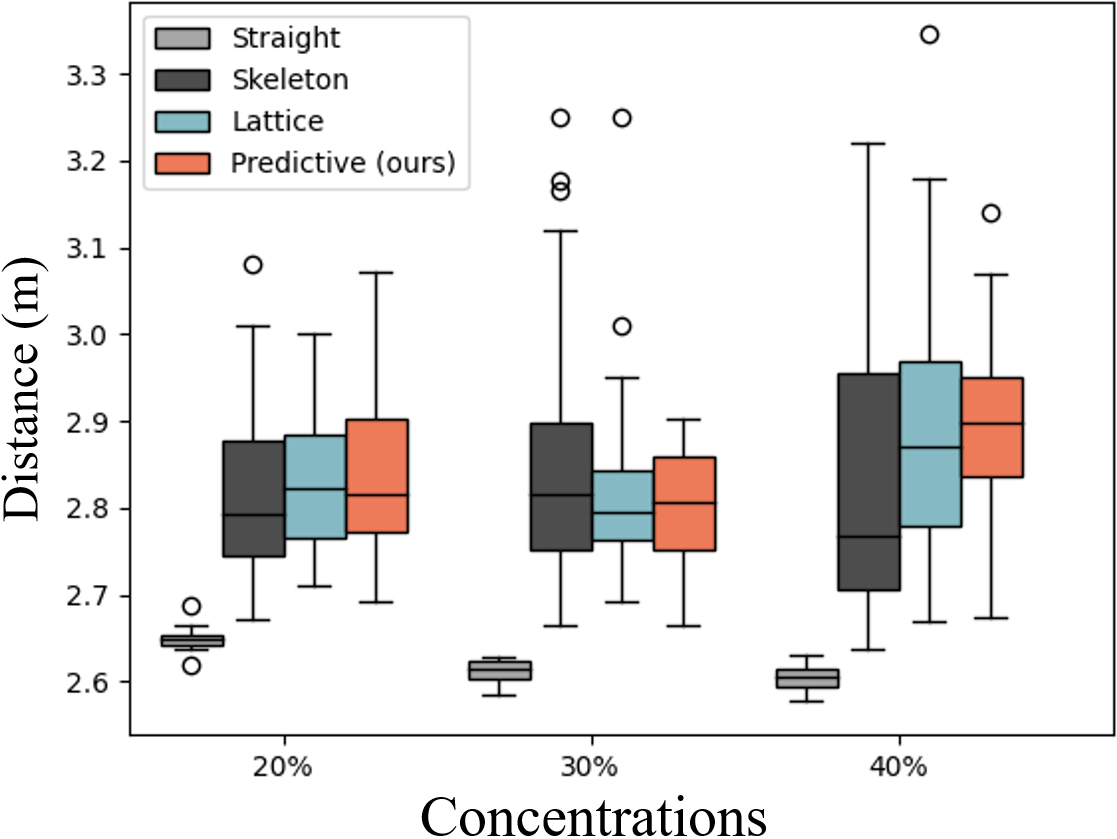}
    \end{minipage}
    \caption{Testbed evaluations across 20\%, 30\%, and 40\% concentrations.}
    \label{fig:eval_result_testbed}
\end{figure}

\section{Conclusion}
\label{sec:conclusion}
We present an integrated framework for ASV navigation in ice-covered waters, combining a deep learning model for ice motion prediction with a lattice planner to plan obstacle-motion-aware paths. Evaluations in simulation and a physical testbed show that predicting the motion of ice significantly improves the navigation performance. The performance gain is especially notable in high ice concentrations with increased ice motion, underscoring the value of anticipating obstacle motions during planning. In future work, we aim to incorporate additional environment factors into our prediction, such ship waves and ocean currents, for more robust inferences.



\end{document}